\makeatletter
\@ifundefined{dcreviewbuild}{\documentclass[preprint,12pt]{elsarticle}}{\documentclass{elsarticle}}
\makeatother

\usepackage{amsmath}
\usepackage{amssymb}
\usepackage{amsthm}
\usepackage{graphicx}
\graphicspath{{../figures/}{./}}
\usepackage{booktabs}
\usepackage[colorlinks=true,allcolors=blue]{hyperref}

\newtheorem{proposition}{Proposition}
\newtheorem{theorem}{Theorem}
\newtheorem{lemma}{Lemma}
\newtheorem{corollary}{Corollary}
\newtheorem{conjecture}{Conjecture}
\theoremstyle{definition}

\newcommand{\NN}{n}
\newcommand{\dep}{L}
\newcommand{\lay}{\ell}
\newcommand{\jet}[1]{\partial_x^{#1}}
\newcommand{\chio}{\chi_1}
\newcommand{\Sig}{\Sigma}
\newcommand{\Ex}{\mathbb{E}}
\newcommand{\Cov}{\operatorname{Cov}}

\newcommand{\poleValidRange}{|\chi_1-1|\le0.01}

\newcommand{\sigwSq}{1.760955}
\newcommand{\sigbSq}{0.05}
\newcommand{\qstar}{0.570048}
\newcommand{\chiTwo}{0.568531}
\newcommand{\chiThree}{2.262727}

\newcommand{\closMaxDevSmallN}{14.79\%}   
\newcommand{\closMaxDevMidN}{3.94\%}      
\newcommand{\closMaxDevLargeN}{2.90\%}    
\newcommand{\expVtwo}{0.6457}
\newcommand{\expVthree}{0.6136}
\newcommand{\expCthirteen}{0.4019}

\newcommand{\bcTerm}{-39.65}
\newcommand{\bcTotal}{134.53}
\newcommand{\bcShare}{29\%}
\newcommand{\aTwoTol}{0.1\%}
\newcommand{\gaussTol}{0.01}
\newcommand{\isserlisA}{1.0002}
\newcommand{\isserlisB}{1.0002}
\newcommand{\isserlisC}{1.0002}
\newcommand{\sigmaOneZero}{+1.3\times10^{-4}}
\newcommand{\rhoTwo}{0.9992}
\newcommand{\rhoOneThree}{1.0001}
\newcommand{\rhoThree}{0.9998}

\newcommand{\slopeOne}{3.6\times10^{-5}\pm5.1\times10^{-5}}
\newcommand{\slopeOneSig}{0.71\sigma}
\newcommand{\slopeTwo}{0.1719\pm0.0089}
\newcommand{\predTwo}{0.1700\pm0.0066}
\newcommand{\devTwo}{0.17\sigma}
\newcommand{\coefThree}{0.2475\pm0.0352}

\newcommand{\predThree}{0.2288\pm0.0134}
\newcommand{\devThree}{0.50\sigma}
\newcommand{\dRatioTol}{5\%}

\newcommand{\plainGrowTwo}{23}
\newcommand{\plainGrowThree}{590}
\newcommand{\plainGrowFour}{1.8\times10^{4}}
\newcommand{\resGrowTwo}{1.39}
\newcommand{\resGrowThree}{1.70}
\newcommand{\resGrowFour}{2.22}

\newcommand{\seedsClosure}{64}
\newcommand{\closureBand}{4\%}

\newcommand{\plateauBand}{3\%}
\newcommand{\driftDelta}{1.1\times10^{-4}}
\newcommand{\driftMismatch}{1.3\sigma}

\newcommand{\zenodoDOI}{10.5281/zenodo.21873633}
\newcommand{\zenodoURL}{https://doi.org/10.5281/zenodo.21873633}

\newcommand{\widthMain}{2^{16}}
\newcommand{\widthMax}{2^{18}}
\newcommand{\depthMain}{256}
\newcommand{\seedsMain}{16}
\newcommand{\fitLo}{64}

\newcommand{\widthAux}{2^{14}}
\newcommand{\seedsAux}{12}

\newcommand{\loglogKtwo}{0.85}
\newcommand{\loglogKthree}{1.80}
\newcommand{\loglogKfour}{2.63}

\newcommand{\vOneAtOne}{1.76}
\newcommand{\vOneAtFour}{0.34}

\newcommand{\lamTransientWindow}{5\times10^{-19}}

\journal{Neural Networks}

\begin{document}

\begin{frontmatter}

\title{Critical initialization destabilizes higher input derivatives in wide scalar-input networks}

%
%
\author[inst1]{Prashant Singh\corref{cor1}\fnref{eq}}
\ead{2023mcb1309@iitrpr.ac.in}
\author[inst1]{Pranav Singh\fnref{eq}}
\ead{2023mcb1308@iitrpr.ac.in}

\affiliation[inst1]{organization={Department of Mathematics, Indian Institute of
                                  Technology Ropar},
                    city={Rupnagar},
                    state={Punjab},
                    postcode={140001},
                    country={India}}

\cortext[cor1]{Corresponding author.}
\fntext[eq]{Both authors contributed equally to this work.}

%
\begin{abstract}
The edge-of-chaos condition preserves first-order input perturbations in wide randomly
initialized networks, but physics-informed losses, score matching and derivative
regularization depend on higher input derivatives.
For smooth scalar-input fully connected networks, using a joint Gaussianity of the finite
derivative jet that holds in the infinite-width limit at each fixed depth, we derive
mean-field recursions through third order that are exact at the variance fixed point, with
finite-depth corrections that decay geometrically. At
criticality, the first-derivative variance is depth-invariant, whereas the
second-derivative variance grows linearly whenever the activation has nonzero curvature.
The resulting third-order system closes on mean-field susceptibilities. For residual
networks with branch scale $L^{-1/2}$, we prove that every fixed finite derivative order
has uniformly bounded variance under explicit regularity assumptions. Simulations verify
the critical growth laws, the residual bound, and the closed recursion. The results concern
initialization, not trained-network performance.
\end{abstract}


\end{frontmatter}

\section{Introduction}
\label{sec:intro}

Physics-informed neural networks minimize a differential-equation residual and therefore
depend on input derivatives of the network \citep{raissi2019}. The same is true of implicit
neural representations regularized by their gradients \citep{sitzmann2020}, and of explicit
score matching \citep{hyvarinen2005}, whose objective contains the Laplacian of the model
log-density and so depends on second input derivatives. Denoising and sliced variants, and
the score parameterization of Song and Ermon~\citep{song2019}, model the score directly and
require only first-order derivatives; by the results below, that is the order which is
structurally well behaved. Derivative conditioning in these settings is not determined by
output variance alone.

These networks do not improve with depth in the way ordinary feedforward networks do.
Accuracy degrades as layers are added, and Wang et al.~\citep{wang2024pirate} attribute the
degradation to initialization rather than to optimization, reporting that standard schemes
leave the network's derivatives poorly conditioned before training begins. Their remedy is
architectural: a residual connection whose branches are attenuated so that the network
begins close to shallow and deepens as training proceeds.

Mean-field initialization theory identifies the edge of chaos, where the first-order
susceptibility equals one and first-order input perturbations are preserved through depth
\citep{poole2016,schoenholz2017}. This criterion does not specify the propagation of
second and higher input derivatives. That gap matters for derivative-dependent objectives
and motivates the present analysis.

For smooth fully connected networks with one input, we propagate the covariance of the
input-derivative jet at initialization. Under a joint-Gaussian jet assumption in the
infinite-width limit, we obtain a closed recursion through third order. The recursion shows
that criticality keeps the first-derivative variance constant but makes the
second-derivative variance grow linearly whenever the activation has nonzero curvature.
It also gives quadratic growth at third order. For residual layers with branch scale
$L^{-1/2}$, we prove uniform boundedness at every fixed finite derivative order under
bounded-derivative assumptions.

This is complementary to the construction of Wang et al.~\citep{wang2024pirate} rather than
an account of it. Their residual scalar is trainable and initialized to zero, so the network
is an exact identity map at initialization and recovers depth during training; the scaling
analyzed here is fixed and does not change. What the two share is that both attenuate the
residual branch at initialization, and the analysis below gives a mean-field account of why
attenuation at that point is what the derivative variances require.

Two boundaries are worth stating at the outset. Everything here concerns the network at
initialization, before any training step; the connection to trained accuracy runs through
conditioning at initialization, which is the mechanism Wang et al.~\citep{wang2022ntk}
analyze for physics-informed losses, and we do not extend it. And the derivation assumes a
smooth activation, which the results of Section~\ref{sec:theory} show is not incidental.

The analysis establishes no ranking of optimization methods, and its direct conclusion is
architectural: independent layerwise initialization cannot simultaneously preserve a nonzero
second input derivative and stabilize its variance at criticality, whereas depth-scaled
residual branches avoid this variance growth. We validate the recursions and residual
prediction with finite-width simulations.

\section{Related work}
\label{sec:related}

\paragraph{Mean-field theory of initialization.}
The framework this paper builds on originates with Poole et al.~\citep{poole2016} and
Schoenholz et al.~\citep{schoenholz2017}, who showed that a randomly initialized deep network admits a
deterministic description in the wide limit: the preactivation variance obeys a
one-dimensional map with a fixed point, and the correlation between two inputs obeys a
second map whose stability determines how far information propagates. The critical setting
of the weight and bias variances is where the derivative of that second map equals one, and
initializing there maximizes trainable depth. Yang and Schoenholz~\citep{yang2017} extended the analysis to
residual architectures, showing that skip connections widen the critical region, and
Hayou et al.~\citep{hayou2019} used it to compare activation functions.

These derivations are complete for the setting they address. A network trained on a loss
that depends on its output requires the output and its first derivative with respect to the
input, the latter because it controls how gradients traverse depth, and both are what the
framework follows. The present paper extends the same machinery to the higher input
derivatives that a different class of loss requires. The object it propagates, the jet
covariance of \eqref{eq:sigma}, contains the two quantities above as its first entries.

\paragraph{Physics-informed networks and their training behavior.}
Raissi et al.~\citep{raissi2019} introduced the formulation in which a differential-equation residual
serves as the training objective. Wang et al.~\citep{wang2021gradient} identify pathologies in the
resulting gradient flow, and Wang et al.~\citep{wang2022ntk} analyze the optimization
through the neural tangent kernel \citep{jacot2018}, showing that the residual and boundary
components of the kernel can differ by orders of magnitude and that this imbalance explains
several observed training failures. That analysis concerns the kernel at initialization and
its consequences for convergence; the present paper concerns the variances of the network's
input derivatives at initialization, which is the quantity entering that kernel rather than
the kernel itself.

Wong et al.~\citep{wong2022sinusoidal} likewise locate a training failure at initialization
rather than in the optimizer. They report that a network of increasing expressiveness is
biased towards flat output functions at initialization, which can come close to satisfying a
differential-equation residual while remaining far from the solution, and their remedy
raises the variability of the input gradient by mapping the input into a sinusoidal space.
That remedy acts on the input encoding, and the derivative order it addresses is the first.
The recursion analyzed here acts through depth and concerns orders two and above, so the two
accounts bear on different orders of the same jet.

Wang et al.~\citep{wang2024pirate} report that physics-informed networks degrade as depth
increases and attribute this to initialization. Their remedy is a residual connection
carrying a trainable scalar initialized to zero, so that each block is an exact identity map
at initialization and the network recovers depth during training, together with a
least-squares initialization of the final layer. This is mechanistically distinct from the
fixed depth-dependent attenuation analyzed in Section~\ref{sec:theory-residual}, which never
changes. The two agree in attenuating the residual branch at initialization and differ in
whether that attenuation is subsequently learned.

\paragraph{Residual branch scaling.}
The exponent $\gamma = 1/2$ is not new, and its properties at orders zero and one are
established. Hayou et al.~\citep{hayou2021stable} show that scaling each residual branch by
$\dep^{-1/2}$ keeps the output variance bounded where an unscaled network gives variance
growing linearly in depth. Zhang et al.~\citep{zhang2019sharp} prove the threshold is sharp: the forward
process is unbounded once the branch scale exceeds $\dep^{-1/2+c}$ for any $c>0$.
Marion et al.~\citep{marion2025scaling} show that $\dep^{-1/2}$ is the distinguished exponent in the
large-depth limit under independent initialization, other choices giving either explosion or
an identity map. Architecturally, Bachlechner et al.~\citep{bachlechner2021rezero} obtain a related effect with
a trainable residual scalar initialized to zero, which is the construction
Wang et al.~\citep{wang2024pirate} adopt.

What Section~\ref{sec:theory-residual} adds is not the exponent but its reach: the same
single choice $\gamma=1/2$ bounds every input-derivative order simultaneously, by an
argument in which the order of the derivative does not appear.

\paragraph{Adjacent questions about higher derivatives.}
Two lines of work involve higher derivatives of neural networks and are worth distinguishing
by what each sets out to do.

Chickering~\citep{ntangentprop2024} addresses the cost of evaluating higher-order derivatives of a deep
network, giving an algorithm quasilinear in the order. That is the computational question:
how expensive the derivatives are to obtain. The question here is a statistical one, namely
how their variances behave as depth grows, and the two are independent. A practitioner needs
both.

Martens et al.~\citep{martens2021dks} shape the kernel map of a deep network by modifying activation
functions, and their analysis differentiates that map with respect to the correlation
between two inputs. The derivatives they take are of the kernel with respect to its
argument; the derivatives here are of the network with respect to its input. Both are
higher-order quantities in a mean-field analysis, and they are different objects, so the
results do not overlap.

\paragraph{Correlator hierarchies through depth.}
Hanin~\citep{hanin2022hierarchies} develops a perturbative hierarchy for the joint cumulants
of a random network and its derivatives, and Roberts et al.~\citep{roberts2022principles} propagate
preactivation correlators together with the neural tangent kernel and its parameter-space
derivatives, with finite-width corrections organized in powers of $1/\NN$. Both are
hierarchies of correlators through depth, and the machinery is close to what is used here.

The scope of the first is worth stating precisely, because it is broader than the recursions
drawn from it. Its cumulant estimates are stated for differential operators in the input of
any order the smoothness of the activation permits, so higher input derivatives lie inside
that framework rather than outside it. What the estimates supply for such an operator is its
order of magnitude in $1/\NN$ at fixed depth, with the layer index carried in the implicit
constant. The explicit layerwise recursions derived from them cover the network value and
its first input derivatives, and the large-depth conclusion drawn from them concerns
gradients with respect to first-layer weights. Neither line of work takes the depth
dependence of a higher input derivative's variance as its object, and that is the quantity
propagated here: the covariance of $\jet{k}h^{\lay}$ for $k$ up to three, at fixed width.

To our knowledge, no prior work isolates the input-derivative correlators $\jet{k}h^{\lay}$
with $k\ge2$ as the propagated object and derives their depth-growth law at criticality, and
we have found no prior notice of the parity failure described in
Section~\ref{sec:method-parity}. Yang and Schoenholz~\citep{yang2017} establish polynomial depth behavior at
criticality for activations and gradients, and the growth law below is in that spirit,
specialized to input derivatives. The search behind these statements covered the mean-field
and neural-tangent-kernel literatures and forward citations of the works above; it was not
exhaustive, which is why the claims are stated in this form.

\section{Setup}
\label{sec:setup}

A fully connected network of width $\NN$ and depth $\dep$ maps a scalar input $x$ to
preactivations $h^\lay \in \mathbb{R}^\NN$ through
\begin{equation}
\label{eq:net}
h^{\lay} = W^{\lay}\phi\!\left(h^{\lay-1}\right) + b^{\lay},
\qquad
W^{\lay}_{ij}\sim\mathcal N\!\left(0,\ \sigma_w^2/\NN\right),
\quad
b^{\lay}_i\sim\mathcal N\!\left(0,\ \sigma_b^2\right),
\end{equation}
with the activation $\phi$ applied componentwise and all weights and biases drawn
independently at initialization. The first layer has fan-in equal to the input dimension
rather than $\NN$, so $W^1_{i}\sim\mathcal N(0,\sigma_w^2)$. Throughout, $\phi=\tanh$
unless stated otherwise, and every statement refers to the network at initialization,
before any training step. The analytical statements below concern fixed finite depth and
fixed finite derivative order as $\NN\to\infty$.

\subsection{Mean-field variance propagation}
\label{sec:setup-mf}

In the limit $\NN\to\infty$ the preactivations at each layer become Gaussian with zero
mean, and their variance $q^\lay = \Ex[(h^\lay_i)^2]$ obeys a deterministic recursion
\citep{poole2016,schoenholz2017}
\begin{equation}
\label{eq:qmap}
q^{\lay} = \sigma_w^2\,\Ex_{z\sim\mathcal N(0,q^{\lay-1})}\!\left[\phi(z)^2\right] + \sigma_b^2 ,
\end{equation}
which has a fixed point $q^\ast$. We call the quantities
\begin{equation}
\label{eq:chi}
\chi_m := \sigma_w^2\,\Ex_{z\sim\mathcal N(0,q^\ast)}\!\left[\phi^{(m)}(z)^2\right],
\qquad m\ge 1,
\end{equation}
the \emph{order-$m$ susceptibilities}. Only $\chi_1$ appears in the classical theory. It
governs how a perturbation of the input propagates: the variance of the input-output
Jacobian is multiplied by $\chi_1$ at each layer, so it grows without bound when
$\chi_1>1$, decays to zero when $\chi_1<1$, and is preserved when $\chi_1=1$. The
condition
\begin{equation}
\label{eq:eoc}
\chi_1 = 1
\end{equation}
defines the \emph{edge of chaos}, the standard prescription for choosing
$(\sigma_w^2,\sigma_b^2)$ so that first-order perturbations neither vanish nor explode.
It is also closely related to stable backward-signal propagation in the classical
mean-field analysis. For $\tanh$ with $\sigma_b^2=\sigbSq$ this gives
$\sigma_w^2=\sigwSq$, $q^\ast=\qstar$, and $\chi_2=\chiTwo$, $\chi_3=\chiThree$.

\subsection{Jets and the jet covariance}
\label{sec:setup-jets}

The quantity a loss actually sees is the scalar network output. We take the standard
independent linear readout
\begin{equation}
\label{eq:readout}
u(x) := \sum_{i=1}^{\NN} a_i\,\phi\!\left(h^{\dep}_i(x)\right),
\qquad
a_i\sim\mathcal N\!\left(0,\ \sigma_a^2/\NN\right)
\ \text{independent of all }W^\lay,b^\lay .
\end{equation}

A loss that constrains a differential operator of order $k$ depends on $\jet{k}u$, not on
$u$ alone. The objects that must propagate are therefore the input derivatives of the
preactivations. We write
\begin{equation}
\label{eq:jetdef}
h^\lay_{(j)} := \jet{j} h^\lay \in \mathbb{R}^{\NN},
\qquad j = 0,1,\dots,K,
\end{equation}
and call the collection $\left(h^\lay_{(0)},\dots,h^\lay_{(K)}\right)$ the \emph{$K$-jet}
of the layer. Index $j=0$ is the preactivation itself. The quantity this paper tracks is
the \emph{jet covariance}
\begin{equation}
\label{eq:sigma}
\Sig^\lay_{jk} := \Ex\!\left[h^\lay_{(j),i}\,h^\lay_{(k),i}\right],
\qquad j,k \in \{0,1,\dots,K\},
\end{equation}
whose diagonal entries $v^\lay_k := \Sig^\lay_{kk}$ are the derivative variances and whose
off-diagonal entries record the correlations between derivative orders. We abbreviate
$c_{jk} := \Sig_{jk}$ for $j,k\ge 1$ and $\sigma_{0k} := \Sig_{0k}$.

Including $j=0$ is not a matter of convenience. The recursion for $\Sig$ contains terms
of the form $\Ex[g(h_{(0)})\,h_{(i)}h_{(j)}h_{(k)}]$ in which a function of the
preactivation multiplies a product of jets, and the couplings $\sigma_{0k}$ are what
prevent those terms from vanishing. Section~\ref{sec:theory} shows that omitting row and
column $0$ makes several such terms evaluate to zero when they are not.

The layer map \eqref{eq:net} propagates the jet by the chain rule. Writing
$g^{\lay-1}_{(j)} := \jet{j}\!\left[\phi(h^{\lay-1})\right]$, we have
$h^\lay_{(j)} = W^\lay g^{\lay-1}_{(j)}$ for $j\ge1$, with the bias contributing only at
$j=0$, and the $g_{(j)}$ follow from Fa\`a di Bruno's formula. For $j\le 3$,
\begin{align}
\label{eq:faa}
g_{(0)} &= \phi(h_{(0)}), \nonumber\\
g_{(1)} &= \phi'\,h_{(1)}, \nonumber\\
g_{(2)} &= \phi''\,h_{(1)}^2 + \phi'\,h_{(2)}, \\
g_{(3)} &= \phi'''\,h_{(1)}^3 + 3\phi''\,h_{(1)}h_{(2)} + \phi'\,h_{(3)}, \nonumber
\end{align}
with every $\phi^{(m)}$ evaluated at $h_{(0)}$ and all products taken componentwise. Since
$W^\lay$ is independent of layer $\lay-1$ and has independent zero-mean entries of
variance $\sigma_w^2/\NN$,
\begin{equation}
\label{eq:sigmap}
\Sig^{\lay}_{jk} = \sigma_w^2\,\Ex\!\left[g^{\lay-1}_{(j)}\,g^{\lay-1}_{(k)}\right]
\qquad (j,k\ge1),
\end{equation}
exactly as an expectation over the random network, for any width. The closed recursions
derived below additionally use the mean-field and Gaussian-jet assumptions. Equation~\eqref{eq:sigmap} is the object of
Section~\ref{sec:theory}: everything reduces to evaluating expectations of products of
the Fa\`a di Bruno terms \eqref{eq:faa}.

\subsection{Assumptions}
\label{sec:setup-assumptions}

The closed recursions use one additional mean-field hypothesis.

\begin{description}
  \item[(A1)] The jet $\left(h^\lay_{(0)},\dots,h^\lay_{(K)}\right)$ is jointly Gaussian
    at initialization.
\end{description}

A1 is the jet analogue of the standard mean-field statement, which asserts Gaussianity for
$j=0$ alone. It is not implied by that statement, since $h_{(2)}$ and $h_{(3)}$ are built
from products through \eqref{eq:faa}. It is, however, available in the literature rather
than being a hypothesis this paper must postulate.

\begin{theorem}[A1 holds in the infinite-width limit at each fixed depth]
\label{prop:gaussianjet}
Fix $K$ and a depth $\lay$, and let $\phi$ be $K$ times differentiable with polynomially
bounded $K$th derivative. Then as the hidden widths tend to infinity with $\lay$ fixed, the
finite-dimensional distributions of $x\mapsto h^\lay(x)$ together with its input derivatives
through order $K$ converge to those of a centered Gaussian process with independent
components. In particular A1 holds in that limit.
\end{theorem}

\begin{proof}
This is Theorem 2.2 of Hanin~\citep{hanin2022hierarchies} specialised to $n_0=1$ and
$r=K$. The hypothesis is satisfied by $\tanh$, whose derivatives of every order are
bounded. Related Gaussian-process limits for fixed collections of network quantities are
obtained by other routes in \citep{yang2019tensor}.
\end{proof}

Two qualifications should be read with Theorem~\ref{prop:gaussianjet}. The limit is
taken at fixed depth as the width grows, so it licenses A1 layer by layer and does not by
itself provide a bound uniform in $\lay$; the recursions below are applied at each depth in
that limit, which is the same regime in which the mean-field variance map itself is derived.
And the convergence is in finite-dimensional distributions, which is what the expectations
in \eqref{eq:sigmap} require and no more. Section~\ref{sec:numerics-assumptions} reports the
measured skewness and excess kurtosis of the jet at the widths actually used, which is the
finite-width question that a limit theorem does not answer.

The derivation does not assume that activation factors are independent of the jet factors
they multiply. Under A1, the required expectations are evaluated by Gaussian integration
by parts and the fixed-point identities in Proposition~\ref{prop:identities}.

That independence is a separate statement. It is named here because it is the natural first
thing to try, because Section~\ref{sec:method-parity} turns on precisely where it fails, and
because Section~\ref{sec:numerics-assumptions} reports where it holds.

\begin{description}
  \item[(A2)] For an activation factor $A$, meaning a product of derivatives of $\phi$
    evaluated at $h^\lay_{(0)}$, and a product $J$ of jet components,
    $\Ex[A\,J] = \Ex[A]\,\Ex[J]$.
\end{description}

A2 is not an assumption of this paper's results. No statement in Section~\ref{sec:theory}
uses it, and it is listed only so that the two sections which examine it have something to
refer to.

\section{Theory}
\label{sec:theory}

Every expectation appearing in \eqref{eq:sigmap} has the form
$\Ex\!\left[g(h_{(0)})\,h_{(i_1)}\cdots h_{(i_r)}\right]$, a function of the preactivation
multiplied by a product of jet components. Under A1 these are computable in closed form.
Two facts do the work: Gaussian integration by parts, and a set of identities that pin the
couplings $\sigma_{0k}$ at the variance fixed point.

\subsection{Two tools}
\label{sec:theory-tools}

\begin{lemma}[Gaussian integration by parts]
\label{lem:stein}
Let $X\sim\mathcal N(0,q)$ and let $g$ be differentiable with $\Ex|g'(X)|<\infty$. Then
$\Ex[Xg(X)] = q\,\Ex[g'(X)]$ and, applying this twice,
\begin{equation}
\label{eq:stein2}
\Ex\!\left[X^2g(X)\right] = q\,\Ex[g(X)] + q^2\,\Ex[g''(X)] .
\end{equation}
\end{lemma}

The second tool is a set of identities relating the couplings $\sigma_{0k}$ to the other
entries of the jet covariance. They are exact at every depth, and they are stated that way
here rather than at the fixed point, because the variance map approaches its fixed point
without reaching it at any finite depth.

\begin{proposition}[Variance identities, exact at every depth]
\label{prop:identities}
Write $q_\lay(x) := \Ex[(h^\lay_i(x))^2]$ for the preactivation variance at layer $\lay$,
and $q_\lay',q_\lay'',q_\lay'''$ for its derivatives in $x$. Then for every $\lay$
\begin{equation}
\label{eq:identities-exact}
\sigma_{01}^{\lay} = \tfrac12q_\lay',
\qquad
\sigma_{02}^{\lay} = -v_1^{\lay} + \tfrac12q_\lay'',
\qquad
\sigma_{03}^{\lay} = -3c_{12}^{\lay} + \tfrac12q_\lay'''.
\end{equation}
In particular, wherever $q_\lay$ is constant in $x$,
\begin{equation}
\label{eq:identities}
\sigma_{01} = 0,
\qquad
\sigma_{02} = -v_1,
\qquad
\sigma_{03} = -3c_{12}.
\end{equation}
\end{proposition}

\begin{proof}
Differentiate $q_\lay = \Ex[h_{(0)}^2]$ in $x$, exchanging differentiation and expectation,
which is permitted because the jet has finite second moments and $\phi$ has bounded
derivatives through the order used. The first derivative gives
$q_\lay' = 2\Ex[h_{(0)}h_{(1)}]$. The second gives
$q_\lay'' = 2\Ex[h_{(1)}^2] + 2\Ex[h_{(0)}h_{(2)}]$. The third gives
$q_\lay''' = 6\Ex[h_{(1)}h_{(2)}] + 2\Ex[h_{(0)}h_{(3)}]$. Rearranging each gives
\eqref{eq:identities-exact}, and setting the derivatives of $q_\lay$ to zero gives
\eqref{eq:identities}.
\end{proof}

The identities \eqref{eq:identities} are therefore not available at finite depth as stated.
What makes them usable is that the discrepancy decays geometrically, at a rate fixed by the
activation and the variance map alone.

\begin{lemma}[The transient is geometric]
\label{lem:transient}
Let $V(q) := \sigma_w^2\,\Ex_{z\sim\mathcal N(0,q)}[\phi(z)^2] + \sigma_b^2$ be the variance
map, so that $q_\lay = V(q_{\lay-1})$, and let $q^\ast$ be an attracting fixed point of $V$
with $\lambda := |V'(q^\ast)| < 1$. Fix $\lambda<\bar\lambda<1$. Then for every
$x$ in a neighbourhood of which $q_1$ is smooth, there are constants $C_m$ with
\begin{equation}
\label{eq:transient}
\left|q_\lay^{(m)}(x)\right| \le C_m\,\bar\lambda^{\,\lay},
\qquad m = 1,2,3,
\qquad
\left|q_\lay - q^\ast\right| \le C_0\,\bar\lambda^{\,\lay}.
\end{equation}
\end{lemma}

\begin{proof}
Since $q^\ast$ is attracting, $q_\lay\to q^\ast$ and $|q_\lay - q^\ast|\le C_0\bar\lambda^\lay$
by the standard linearization estimate for a smooth map at an attracting fixed point. The
chain rule gives $q_\lay' = V'(q_{\lay-1})\,q_{\lay-1}'$, so
$q_\lay' = q_1'\prod_{m=1}^{\lay-1}V'(q_m)$; each factor is at most $\bar\lambda$ in modulus
for $\lay$ beyond a finite index, which gives the case $m=1$. Differentiating again,
$q_\lay'' = V''(q_{\lay-1})(q_{\lay-1}')^2 + V'(q_{\lay-1})q_{\lay-1}''$, an affine
recursion in $q_\lay''$ with multiplier bounded by $\bar\lambda$ and forcing
$O(\bar\lambda^{2\lay})$; summing the resulting geometric series gives the case $m=2$. The
case $m=3$ follows in the same way, its forcing being $O(\bar\lambda^{2\lay})$ again.
\end{proof}

The correlation-map slope $\chio$ and the variance-map slope $\lambda$ are distinct:
$\chio$ governs first-order signal propagation, whereas $\lambda$ governs the transient in
the variance map.

\begin{corollary}
\label{cor:indep}
Under A1, write $h_{(k)} = \beta_k^{\lay}h_{(0)} + r_k$ with
$\beta_k^{\lay} := \sigma_{0k}^{\lay}/q_\lay$. The residual $r_k$ is uncorrelated with
$h_{(0)}$ by construction, so under A1 it is independent of it, and $h_{(0)}$ is
independent of the whole collection $\left(r_1,r_2,r_3,\dots\right)$. This decomposition is
exact at every depth. By \eqref{eq:identities-exact},
$\beta_1^{\lay} = q_\lay'/(2q_\lay)$, which vanishes wherever $q_\lay$ is constant in $x$
and is $O(\bar\lambda^{\lay})$ in general by Lemma~\ref{lem:transient}; in that case
$h_{(1)} = r_1$ is independent of $h_{(0)}$.
\end{corollary}

Corollary~\ref{cor:indep} is what makes the expectations tractable. Any
$\Ex[g(h_{(0)})\,h_{(i_1)}\cdots h_{(i_r)}]$ is expanded by substituting
$h_{(k)}=\beta_kh_{(0)}+r_k$, after which every term factorizes into a moment of
$g$ against a power of $h_{(0)}$, evaluated by Lemma~\ref{lem:stein}, times a joint moment
of $\left(h_{(1)},r_2,r_3\right)$, evaluated by Isserlis. No independence assumption beyond
A1 is used.

\subsection{Orders one and two}
\label{sec:theory-order12}

The order-1 recursion is classical: $g_{(1)}=\phi'h_{(1)}$ with $h_{(1)}$ independent of
$h_{(0)}$ by Corollary~\ref{cor:indep}, so \eqref{eq:sigmap} gives
\begin{equation}
\label{eq:v1}
v_1^{\lay} = \chio\,v_1^{\lay-1}.
\end{equation}
Condition \eqref{eq:eoc} makes $v_1$ independent of depth, which is the entire content of
the edge-of-chaos criterion.

\begin{proposition}[Order two]
\label{prop:order2}
Under A1, wherever the preactivation variance is constant in $x$,
\begin{equation}
\label{eq:v2}
v_2^{\lay} = \chio\,v_2^{\lay-1} + 3\chi_2\left(v_1^{\lay-1}\right)^2
\end{equation}
holds in the mean-field limit with no correction terms: the two non-vanishing corrections
cancel identically rather than approximately. At finite depth the same evaluation gives
\begin{equation}
\label{eq:v2-finite}
v_2^{\lay} = \chio^{(\lay)}\,v_2^{\lay-1} + 3\chi_2^{(\lay)}\left(v_1^{\lay-1}\right)^2
+ E_2^{\lay},
\qquad
\left|E_2^{\lay}\right| \le C\bar\lambda^{\,\lay}\left(1 + v_2^{\lay-1}\right),
\end{equation}
where $\chio^{(\lay)},\chi_2^{(\lay)}$ are the susceptibilities evaluated at $q_{\lay-1}$
rather than $q^\ast$ and satisfy $\chio^{(\lay)} = \chio + O(\bar\lambda^{\lay})$,
$\chi_2^{(\lay)} = \chi_2 + O(\bar\lambda^{\lay})$.
\end{proposition}

\begin{proof}
From \eqref{eq:faa}, $g_{(2)} = \phi''h_{(1)}^2 + \phi'h_{(2)}$, so
$v_2^\lay = \sigma_w^2\left(T_A + T_B + T_C\right)$ with
$T_A = \Ex[\phi''^2h_{(1)}^4]$, $T_B = 2\Ex[\phi'\phi''h_{(1)}^2h_{(2)}]$ and
$T_C = \Ex[\phi'^2h_{(2)}^2]$.

For $T_A$, Corollary~\ref{cor:indep} gives $h_{(1)}\perp h_{(0)}$, so
$T_A = \Ex[\phi''^2]\,\Ex[h_{(1)}^4] = 3\Ex[\phi''^2]v_1^2$, contributing $3\chi_2v_1^2$.
This step is exact and does not use activation-jet factorization.

For $T_B$, substitute $h_{(2)}=\beta_2h_{(0)}+r_2$. The $r_2$ term carries
$\Ex[\phi'\phi'']$, which vanishes because $\phi'\phi''$ is odd for odd $\phi$, and the
$h_{(0)}$ term factorizes by Corollary~\ref{cor:indep}:
$T_B = 2\beta_2\,\Ex[\phi'\phi''h_{(0)}]\,v_1$. Since
$\phi'\phi''=\tfrac12\left(\phi'^2\right)'$, Lemma~\ref{lem:stein} gives
$\Ex[\phi'\phi''h_{(0)}] = \tfrac{q^\ast}{2}\Ex[(\phi'^2)'']$, and with
$\beta_2 = \sigma_{02}/q^\ast = -v_1/q^\ast$ from \eqref{eq:identities},
\begin{equation}
\label{eq:TB}
\sigma_w^2\,T_B = -\sigma_w^2\,v_1^2\,\Ex\!\left[(\phi'^2)''\right].
\end{equation}

For $T_C$, the same substitution and \eqref{eq:stein2} give
$T_C = \Ex[\phi'^2]v_2 + \beta_2^2\Cov\!\left(\phi'^2,h_{(0)}^2\right)$ with
$\Cov(\phi'^2,h_{(0)}^2) = q^{\ast2}\Ex[(\phi'^2)'']$, so
\begin{equation}
\label{eq:TC}
\sigma_w^2\,T_C = \chio\,v_2 + \sigma_w^2\,v_1^2\,\Ex\!\left[(\phi'^2)''\right].
\end{equation}

Equations \eqref{eq:TB} and \eqref{eq:TC} are equal and opposite, and \eqref{eq:v2}
follows.

For \eqref{eq:v2-finite}, repeat the evaluation with $q_{\lay-1}$ in place of $q^\ast$ and
$\beta_k^{\lay}$ in place of its fixed-point value. Three changes appear. The
susceptibilities are Gaussian expectations against $\mathcal N(0,q_{\lay-1})$, and
$q\mapsto\sigma_w^2\Ex_{\mathcal N(0,q)}[\phi^{(m)2}]$ is smooth, so they differ from
$\chi_m$ by $O(|q_{\lay-1}-q^\ast|) = O(\bar\lambda^{\lay})$. The term $T_A$ acquires a
contribution through $\beta_1^{\lay} = O(\bar\lambda^{\lay})$, since $h_{(1)}$ is no longer
exactly independent of $h_{(0)}$. And $\beta_2^{\lay}$ differs from $-v_1/q^\ast$ by
$q_{\lay}''/(2q_\lay) = O(\bar\lambda^{\lay})$, so the cancellation between \eqref{eq:TB}
and \eqref{eq:TC} is exact up to that order. Each contribution is bounded by a constant
times $\bar\lambda^{\lay}$, multiplied by the jet moments it accompanies, which gives the
stated bound.
\end{proof}

\subsection{Order three}
\label{sec:theory-order3}

\begin{proposition}[Order three]
\label{prop:order3}
Under A1,
\begin{equation}
\label{eq:v3}
v_3^{\lay} = \chio\,v_3^{\lay-1}
+ 15\chi_3v_1^3
+ 9\chi_2\!\left(v_1v_2 + 2c_{12}^2\right)
- 6\chi_2\,v_1c_{13},
\end{equation}
with all right-hand quantities at layer $\lay-1$.
\end{proposition}

\begin{proof}[Proof sketch]
Squaring $g_{(3)} = \phi'''h_{(1)}^3 + 3\phi''h_{(1)}h_{(2)} + \phi'h_{(3)}$ gives six
terms. Each is evaluated by the substitution of Corollary~\ref{cor:indep} together with
Lemma~\ref{lem:stein} and Isserlis. Writing
$\kappa_{13} := \sigma_w^2\Ex[\phi'''\phi']$ and using
$(\phi'^2)''=2(\phi''^2+\phi'\phi''')$ and $(\phi''^2)''=2(\phi'''^2+\phi''\phi'''')$,
\begin{samepage}
the six contributions below. Each line is multiplied by $\sigma_w^2$:
\begin{align*}
\phi'''^2h_{(1)}^6
&:\quad 15\chi_3v_1^3, \\
9\phi''^2h_{(1)}^2h_{(2)}^2
&:\quad 9\chi_2\left(v_1v_2+2c_{12}^2\right)
      +9\sigma_w^2v_1^3\Ex\!\left[(\phi''^2)''\right], \\
\phi'^2h_{(3)}^2
&:\quad \chio v_3+9\sigma_w^2c_{12}^2\Ex\!\left[(\phi'^2)''\right], \\
6\phi'''\phi''h_{(1)}^4h_{(2)}
&:\quad -9\sigma_w^2v_1^3\Ex\!\left[(\phi''^2)''\right], \\
2\phi'''\phi'h_{(1)}^3h_{(3)}
&:\quad 6\kappa_{13}v_1c_{13}, \\
6\phi''\phi'h_{(1)}h_{(2)}h_{(3)}
&:\quad -3\sigma_w^2\Ex\!\left[(\phi'^2)''\right]
      \left(v_1c_{13}+3c_{12}^2\right).
\end{align*}
\end{samepage}
The $\Ex[(\phi''^2)'']$ terms in the second and fourth lines cancel, as do the
$c_{12}^2$ terms in the third and sixth lines. The $v_1c_{13}$ terms in the fifth and sixth
lines combine through
$\sigma_w^2\Ex[(\phi'^2)''] = 2\chi_2+2\kappa_{13}$ into
$\left(6\kappa_{13}-3(2\chi_2+2\kappa_{13})\right)v_1c_{13} = -6\chi_2v_1c_{13}$, so
$\kappa_{13}$ does not appear in \eqref{eq:v3}.
\end{proof}

\subsection{The off-diagonals, and closure}
\label{sec:theory-offdiag}

Propositions~\ref{prop:order2} and \ref{prop:order3} involve $c_{12}$ and $c_{13}$, so the
system is not yet closed. Both follow from the same evaluation.

\begin{proposition}[Off-diagonal recursions]
\label{prop:offdiag}
Under A1,
\begin{equation}
\label{eq:offdiag}
c_{12}^{\lay} = \chio\,c_{12}^{\lay-1},
\qquad
c_{13}^{\lay} = \chio\,c_{13}^{\lay-1} - 3\chi_2\left(v_1^{\lay-1}\right)^2 .
\end{equation}
\end{proposition}

\begin{proof}
For $c_{12}=\sigma_w^2\Ex[g_{(1)}g_{(2)}]$, the term $\Ex[\phi'\phi''h_{(1)}^3]$ vanishes
twice over: $\Ex[\phi'\phi'']=0$ by parity, and $\Ex[h_{(1)}^3]=0$ because $h_{(1)}$ is
independent of $h_{(0)}$ and Gaussian. In $\Ex[\phi'^2h_{(1)}h_{(2)}]$, the substitution
$h_{(2)}=\beta_2h_{(0)}+r_2$ leaves a $\beta_2$ term carrying $\Ex[h_{(1)}]=0$ and a
residual term $\Ex[\phi'^2]\Ex[h_{(1)}r_2] = \Ex[\phi'^2]c_{12}$.

For $c_{13}=\sigma_w^2\Ex[g_{(1)}g_{(3)}]$ there are three terms.
$\sigma_w^2\Ex[\phi'\phi'''h_{(1)}^4] = 3\kappa_{13}v_1^2$ by independence and
$\Ex[h_{(1)}^4]=3v_1^2$. The middle term is three halves of $T_B$ from
Proposition~\ref{prop:order2}, so by \eqref{eq:TB} and
$\sigma_w^2\Ex[(\phi'^2)'']=2\chi_2+2\kappa_{13}$ it equals $-3v_1^2(\chi_2+\kappa_{13})$.
The third is $\chio c_{13}$ by the substitution above. The $\kappa_{13}$ contributions
cancel between the first two.
\end{proof}

\begin{corollary}[Closure]
\label{cor:closure}
The quantities $\left(v_1,v_2,v_3,c_{12},c_{13}\right)$ satisfy a closed autonomous system
whose coefficients are $\chio$, $\chi_2$ and $\chi_3$ alone.
\end{corollary}

At criticality the solution is explicit: $v_1$ and $c_{12}$ are depth-independent, $v_2$ and
$c_{13}$ are linear with equal and opposite slopes $\pm3\chi_2v_1^2$, so that
\begin{equation}
\label{eq:invariant}
v_2^{\lay} + c_{13}^{\lay} \quad\text{is independent of depth,}
\end{equation}
and $v_3$ is quadratic with leading coefficient $\tfrac{45}{2}\chi_2^2v_1^3$ once the
depth-independent forcing terms become subleading.

\begin{corollary}[Two-sidedness]
\label{cor:twosided}
At criticality, $c_{13}^{\lay} = c_{13}^{0} - 3\chi_2v_1^2\lay$ decreases monotonically
whenever $\chi_2>0$. Beyond a finite depth $c_{13}<0$, so the term $-6\chi_2v_1c_{13}$ in
\eqref{eq:v3} is positive and remains so. Every forcing term in \eqref{eq:v3} is then
positive, and the growth $v_3^{\dep}=\Theta(\dep^2)$ holds as a two-sided bound rather
than only from above.
\end{corollary}

\subsection{The recursion closes on the susceptibilities}
\label{sec:theory-pattern}

Individual terms in Proposition~\ref{prop:order3} involve activation moments outside the
family \eqref{eq:chi}: $\kappa_{13}=\sigma_w^2\Ex[\phi'''\phi']$,
$\sigma_w^2\Ex[(\phi'^2)'']$ and $\sigma_w^2\Ex[(\phi''^2)'']$. None survives. The same
happens at order two, where $\Ex[(\phi'^2)'']$ enters two terms and cancels between them.
In both cases the surviving coefficients are susceptibilities $\chi_m$ alone.

\begin{conjecture}
\label{conj:closure}
For every $k$, the recursion for $v_k$ obtained from \eqref{eq:sigmap} has coefficients
that are polynomial in the susceptibilities $\chi_1,\dots,\chi_k$ and involves no other
moment of the activation.
\end{conjecture}

The evidence is orders one through three, where every non-susceptibility moment cancels in
pairs, and the pairings are between a squared term and the cross term that shares its
activation factor. We have not established the general case, but the cancellations have a
common cause which is worth stating, because it identifies what a general proof would need.

Each cancellation consumes one of the identities \eqref{eq:identities}. Order two uses
$\sigma_{01}=0$ and $\sigma_{02}=-v_1$; order three uses those together with
$\sigma_{03}=-3c_{12}$. These are not independent facts. They are the single statement
$\mathrm{d}^mq^\ast/\mathrm{d}x^m=0$, one instance per $m$, and they are available only
because the variance fixed point is independent of the input.

The couplings $\sigma_{0k}$ are also the only route by which a moment outside the family
\eqref{eq:chi} can enter. Dependence on $h_{(0)}$ reaches an expectation
$\Ex[g(h_{(0)})h_{(i_1)}\cdots h_{(i_r)}]$ solely through the coefficients
$\beta_k=\sigma_{0k}/q^\ast$ of Corollary~\ref{cor:indep}: with those set to zero every
such expectation would factorize into $\Ex[g]$ times a jet moment, and $\Ex[g]$ is a
susceptibility whenever $g$ is a square.

What pairs the terms is the identity
$\phi^{(m)}\phi^{(m+1)} = \tfrac12\left(\phi^{(m)2}\right)'$. The cross term between
adjacent derivative orders carries $\phi^{(m)}\phi^{(m+1)}$, so Lemma~\ref{lem:stein}
converts it into a moment of $\left(\phi^{(m)2}\right)''$; and the same
$\left(\phi^{(m)2}\right)''$ is what Lemma~\ref{lem:stein} produces as the correction to
$\Ex[\phi^{(m)2}h_{(j)}h_{(k)}]$ when $h_{(j)}$ or $h_{(k)}$ is regressed on $h_{(0)}$. The
two therefore carry the same activation moment, and their $\beta$ factors are fixed by
\eqref{eq:identities}. At orders two and three those factors are exactly the ones that
cancel.

A proof of Conjecture~\ref{conj:closure} would have to show that this persists: that
$\mathrm{d}^kq^\ast/\mathrm{d}x^k=0$ delivers, at every order, the coefficient that cancels
the Stein correction against the adjacent cross term. We do not attempt it here.

\subsection{Growth at criticality}
\label{sec:theory-growth}

Away from criticality the order-1 recursion \eqref{eq:v1} gives
$v_1^{\lay} = A\,\chio^{\lay}$ for an amplitude $A$ fixed by the first few layers, so the
forcing in \eqref{eq:v2} is proportional to $\chio^{2\lay}$ and the order-2 solution is a
sum of two geometric terms,
\begin{equation}
\label{eq:v2sol}
v_2^{\lay} = C\,\chio^{\lay} + D\,\chio^{2\lay},
\qquad
D = \frac{3\chi_2A^2}{\chio\left(\chio-1\right)} .
\end{equation}
The homogeneous root is $\chio$ and the particular root is $\chio^2$. These are distinct
for every $\chio\ne1$, and they collide exactly at $\chio=1$.

That collision is the mechanism behind the growth law. Depth-invariance of $v_1$ requires
$\chio=1$ by \eqref{eq:v1}; but at $\chio=1$ the two roots coincide, the decomposition
\eqref{eq:v2sol} degenerates, and the resonant solution is linear in depth:
\begin{equation}
\label{eq:v2crit}
v_2^{\dep} = v_2^{0} + 3\chi_2v_1^2\,\dep .
\end{equation}
The single condition that makes the first derivative propagate without gain or loss is
therefore the same condition that makes the second derivative grow without bound. There is
no value of $\chio$ at which both are stationary.

Equation \eqref{eq:v2sol} also makes $D$ a prediction rather than a fitted quantity: $\chi_2$
follows from the activation and the fixed point through \eqref{eq:chi}, and $A$ is measured
from $v_1$ alone. Its $1/(\chio-1)$ pole is the quantitative signature of the collision,
and Section~\ref{sec:numerics} tests it with no free parameter on the theory side.

At order three the same solution structure gives a bound in both directions. For the upper
bound, suppose $v_3^{\lay}\sim K\lay^{\,p}$. The forcing in \eqref{eq:v3} is then
$O(\lay) + O(\lay^{\,p/2})$, since $v_1$ and $c_{12}$ are depth-independent at criticality,
$v_2$ is linear, and $\lvert c_{13}\rvert \le \sqrt{v_1v_3}$ by Cauchy--Schwarz. Summing
over depth gives $v_3^{\dep} = O(\dep^2) + O(\dep^{\,p/2+1})$, so consistency requires
$p = \max\left(2,\ p/2+1\right)$, whose only solution is $p=2$. This excludes $p>2$ rather
than merely permitting $p=2$. The lower bound is Corollary~\ref{cor:twosided}, so
\begin{equation}
\label{eq:v3growth}
v_3^{\dep} = \Theta\!\left(\dep^2\right),
\qquad\text{with leading coefficient}\quad \tfrac{45}{2}\chi_2^2v_1^3,
\end{equation}
the coefficient following from \eqref{eq:v3} once $9\chi_2v_1v_2$ and $-6\chi_2v_1c_{13}$
become the dominant forcing terms, contributing $27$ and $18$ respectively.

Equations \eqref{eq:v2crit} and \eqref{eq:v3growth} were obtained from the fixed-point form
of the recursions. Since that form is not available at any finite depth, the growth laws are
not yet established for the network of Section~\ref{sec:setup}; what follows closes that gap.
The conclusion is that the transient perturbs the recursion multiplicatively by a summable
amount, which changes the constants but not the powers.

\begin{theorem}[Critical growth, with the transient controlled]
\label{thm:asymptotic}
Assume A1, let $\phi$ have bounded derivatives through order four, and let the variance map
have an attracting fixed point $q^\ast>0$ with $\lambda=|V'(q^\ast)|<1$ as in
Lemma~\ref{lem:transient}. At criticality $\chio=1$, with $v_1^\ast := \lim_\lay v_1^\lay$,
\begin{equation}
\label{eq:asymptotic}
v_2^{\dep} = 3\chi_2\left(v_1^\ast\right)^2\dep + O(1),
\qquad
v_3^{\dep} = \tfrac{45}{2}\chi_2^2\left(v_1^\ast\right)^3\dep^{2} + O(\dep).
\end{equation}
\end{theorem}

\begin{proof}
Fix $\bar\lambda\in(\lambda,1)$ and let $\epsilon_\lay$ denote any sequence with
$|\epsilon_\lay|\le C\bar\lambda^{\lay}$, the constant changing from line to line.

Write the order-1 recursion as $v_1^{\lay} = (1+\epsilon_\lay)v_1^{\lay-1}$, using
$\chio^{(\lay)} = \chio + \epsilon_\lay = 1 + \epsilon_\lay$ from
Proposition~\ref{prop:order2}. Because $\sum_\lay|\epsilon_\lay|<\infty$, the product
$\prod_{m\le\lay}(1+\epsilon_m)$ converges to a finite nonzero limit, so $v_1^\ast$ exists
and $v_1^{\lay} = v_1^\ast + \epsilon_\lay$. The same argument applied to
$c_{12}^{\lay} = (1+\epsilon_\lay)c_{12}^{\lay-1} + \epsilon_\lay$ gives a bounded
$c_{12}^{\lay}$.

Now put \eqref{eq:v2-finite} in the form $v_2^{\lay} = (1+\alpha_\lay)v_2^{\lay-1} + \Phi +
\rho_\lay$ with $\Phi := 3\chi_2(v_1^\ast)^2$ and $|\alpha_\lay|,|\rho_\lay|\le
C\bar\lambda^{\lay}$; the $v_2^{\lay-1}$ factor in the error bound of
\eqref{eq:v2-finite} is absorbed into $\alpha_\lay$ and the remainder into $\rho_\lay$.
Variation of constants gives
\begin{equation}
\label{eq:varconst}
v_2^{\dep} = P_{0,\dep}\,v_2^{0}
+ \sum_{\lay=1}^{\dep} P_{\lay,\dep}\left(\Phi + \rho_\lay\right),
\qquad
P_{\lay,\dep} := \prod_{m=\lay+1}^{\dep}\left(1+\alpha_m\right).
\end{equation}
Since $\sum_m|\alpha_m| \le C\bar\lambda/(1-\bar\lambda) =: S <\infty$, every $P_{\lay,\dep}$
lies in $[e^{-S},e^{S}]$, and
$\left|P_{\lay,\dep}-1\right| \le e^{S}\sum_{m>\lay}|\alpha_m| \le C'\bar\lambda^{\lay}$.
Hence $\sum_{\lay\le\dep}P_{\lay,\dep}\Phi = \Phi\dep + O(1)$, while
$\sum_{\lay\le\dep}P_{\lay,\dep}\rho_\lay = O(1)$ and $P_{0,\dep}v_2^{0} = O(1)$. This is
the first claim.

For the second, the same treatment of the $c_{13}$ recursion in \eqref{eq:offdiag} gives
$c_{13}^{\lay} = -3\chi_2(v_1^\ast)^2\lay + O(1)$. Substituting this and
$v_2^{\lay}=\Phi\lay+O(1)$ into \eqref{eq:v3} makes the forcing
$9\chi_2v_1v_2 = 27\chi_2^2(v_1^\ast)^3\lay + O(1)$ and
$-6\chi_2v_1c_{13} = 18\chi_2^2(v_1^\ast)^3\lay + O(1)$, the remaining terms being $O(1)$.
Applying \eqref{eq:varconst} with this forcing and
$\sum_{\lay\le\dep}\lay = \dep^2/2 + O(\dep)$ gives the second claim.
\end{proof}

Everything so far concerns the preactivations. The quantity a derivative-dependent loss
evaluates is the network output, and the two are related by a constant.

\begin{corollary}[Transfer to the network output]
\label{cor:readout}
For the readout \eqref{eq:readout}, with $a$ independent of all other parameters, and for
every $k\ge1$,
\begin{equation}
\label{eq:readout-transfer}
\Ex\!\left[\left(\jet{k}u\right)^2\right]
= \frac{\sigma_a^2}{\sigma_w^2}\;v_k^{\dep+1}.
\end{equation}
Consequently every statement proved above for $v_k$ holds verbatim for
$\Ex[(\jet{k}u)^2]$, up to that fixed constant: at criticality $\Ex[(\jet{1}u)^2]$ is
depth-independent, $\Ex[(\jet{2}u)^2]$ grows linearly, $\Ex[(\jet{3}u)^2]$ grows
quadratically, and under the residual scaling of Proposition~\ref{prop:residual} all of them
are bounded uniformly in depth.
\end{corollary}

\begin{proof}
Differentiating \eqref{eq:readout} $k\ge1$ times in $x$ gives
$\jet{k}u = \sum_i a_i\,g^{\dep}_{(k),i}$, since
$\jet{k}\!\left[\phi(h^{\dep}_i)\right] = g^{\dep}_{(k),i}$ by definition. Since $a$ is
independent of $\left(W^\lay,b^\lay\right)$ and hence of $g^{\dep}$, with
$\Ex[a_ia_j]=\delta_{ij}\sigma_a^2/\NN$,
\[
\Ex\!\left[\left(\jet{k}u\right)^2\right]
= \sum_{i,j}\Ex[a_ia_j]\,\Ex\!\left[g^{\dep}_{(k),i}g^{\dep}_{(k),j}\right]
= \frac{\sigma_a^2}{\NN}\sum_i\Ex\!\left[\left(g^{\dep}_{(k),i}\right)^2\right]
= \sigma_a^2\,\Ex\!\left[\left(g^{\dep}_{(k),1}\right)^2\right],
\]
the last step by exchangeability of the neurons. By \eqref{eq:sigmap},
$v_k^{\dep+1} = \sigma_w^2\Ex[(g^{\dep}_{(k),1})^2]$ for $k\ge1$, which gives
\eqref{eq:readout-transfer}.
\end{proof}

\subsection{The dichotomy}
\label{sec:theory-dichotomy}

The growth in \eqref{eq:v2crit} is driven by $\chi_2$, so it can be removed by setting
$\chi_2 = 0$. That escape closes off the property the network needs.

\begin{proposition}[Dichotomy]
\label{prop:dichotomy}
Assume A1 for the scalar-input network, let $\phi$ be locally Lipschitz and twice
differentiable in the ordinary sense at Lebesgue-almost every point, with $\phi''$ locally
bounded where it exists, and let $\chio=1$. Then exactly one of the following holds.
\begin{enumerate}
  \item $\chi_2 > 0$, and $v_2^{\dep}$ grows without bound in $\dep$.
  \item $\chi_2 = 0$, and $\jet{2}u = 0$ almost everywhere, for every realization.
\end{enumerate}
Within this scalar-input, mean-field, Gaussian-jet setting, a fully connected network with
independent layerwise initialization therefore cannot have both a nonvanishing second input
derivative and a depth-stable second-derivative variance.
\end{proposition}

\begin{proof}
By \eqref{eq:chi}, $\chi_2 = \sigma_w^2\Ex[\phi''(z)^2]$ with $z$ Gaussian of variance
$q^\ast>0$. That Gaussian has full support, so $\chi_2=0$ forces $\phi''=0$ at almost every
point where it exists, hence almost everywhere. By \eqref{eq:faa} the second jet component
is $g_{(2)} = \phi''h_{(1)}^2 + \phi'h_{(2)}$, whose first term then vanishes almost
everywhere. At the input layer $h^1 = W^1x + b^1$ gives $h^1_{(2)}=0$ identically, and if
$h^{\lay-1}_{(2)}=0$ almost everywhere then $g_{(2)}^{\lay-1}=0$ and hence
$h^{\lay}_{(2)}=0$ almost everywhere. Induction on depth gives $\jet{2}u=0$ almost
everywhere. The argument does not require $\phi$ to be affine, and deliberately so: a
piecewise-linear $\phi$ has $\phi''=0$ almost everywhere without being affine, and the
conclusion holds for it. If instead $\chi_2>0$, Theorem~\ref{thm:asymptotic} makes
$v_2^{\dep}$ grow linearly in $\dep$.
\end{proof}

The second horn is not hypothetical, and the hypothesis above is stated pointwise precisely
so that it covers the case that matters. A piecewise-linear activation is locally Lipschitz
and twice differentiable away from its finitely many kinks, where $\phi''=0$; it therefore
satisfies the hypothesis, and a network built from it has $\jet{2}u = 0$ almost everywhere.
Two qualifications keep that statement literal. Such a $\phi$ is \emph{not} twice weakly
differentiable in the usual sense, because its distributional second derivative is a sum of
point masses at the kinks rather than a locally integrable function; a hypothesis phrased
that way would exclude the very example the horn is about. And what fails for the resulting
network is the ordinary second derivative, which vanishes almost everywhere. Such a network
can still approximate a solution of a second-order equation in a norm that does not see
$\jet{2}u$ pointwise, and can be used with a weak or variational formulation of the
residual; what it cannot do is satisfy a strong-form residual that evaluates $\jet{2}u$
pointwise, since that quantity is zero almost everywhere regardless of the parameters.

Nor is the first horn confined to $\tanh$. A sinusoidal activation, used in implicit neural
representations \citep{sitzmann2020}, has $\phi''=-\sin$ and therefore $\chi_2>0$, so a
network built from it falls in the first branch and its second-derivative variance grows
without bound in depth. The initialization scheme accompanying that architecture controls
the distribution of the activations themselves, which is order zero in the present
terminology; Proposition~\ref{prop:dichotomy} concerns order two and is not addressed by it.

Proposition~\ref{prop:dichotomy} is what distinguishes this failure from an ordinary
tuning problem. The edge-of-chaos condition \eqref{eq:eoc} is a choice of
$(\sigma_w^2,\sigma_b^2)$, and one might expect a different choice to control the second
derivative as well. Section~\ref{sec:theory-growth} rules that out at fixed $\chi_2$, and
the dichotomy rules out obtaining it by changing $\chi_2$. The obstruction is in the
architecture, not in the hyperparameters.

\subsection{Residual scaling}
\label{sec:theory-residual}

Since Proposition~\ref{prop:dichotomy} places the obstruction in the architecture, the
remedy has to be architectural. For the residual network, let $h^0\in\mathbb R^\NN$ be a
finite-variance learned or random embedding of the scalar input; this makes the skip
connection dimensionally well defined. Replacing \eqref{eq:net} by a residual layer with
branch scale $\dep^{-\gamma}$,
\begin{equation}
\label{eq:resnet}
h^{\lay} = h^{\lay-1} + \dep^{-\gamma}\left(W^{\lay}\phi\!\left(h^{\lay-1}\right) + b^{\lay}\right),
\end{equation}
is enough at $\gamma = 1/2$ for each fixed finite derivative order under the conditions
given next.

\begin{proposition}[Residual scaling]
\label{prop:residual}
Fix $K$. Assume A1, that the derivatives of $\phi$ through order $K$ are bounded, and that
the initial jet has finite second moments. Then, for the network \eqref{eq:resnet} with
$\gamma = 1/2$, each $v_k^{\dep}$ with $k\le K$ is bounded uniformly in $\dep$. The bound
may depend on $K$, the activation, and the initial jet.
\end{proposition}

\begin{proof}
Differentiating \eqref{eq:resnet} $k$ times gives
$h^{\lay}_{(k)} = h^{\lay-1}_{(k)} + \dep^{-1/2}W^{\lay}g^{\lay-1}_{(k)}$ for $k\ge1$.
Since $W^{\lay}$ has mean zero and is independent of layer $\lay-1$, the cross term
vanishes in expectation and
\begin{equation}
\label{eq:resrec}
v_k^{\lay} = v_k^{\lay-1} + \frac{\sigma_w^2}{\dep}\,
\Ex\!\left[\left(g^{\lay-1}_{(k)}\right)^2\right].
\end{equation}
Write $g_{(k)}=\phi' h_{(k)}+P_k$, where $P_k$ is a finite sum of products of lower-order
jet components with bounded activation derivatives. For $k=1$, \eqref{eq:resrec} and
boundedness of $\phi'$ give $v_1^{\lay}\le(1+B_1/\dep)v_1^{\lay-1}$ for a constant $B_1$.
Suppose the derivative variances below order $k$ are bounded. Under A1, their joint
Gaussian moments are then bounded by Isserlis' formula. Cauchy--Schwarz and Young's
inequality therefore give constants $B_k,M_k<\infty$, independent of $\lay$ and $\dep$,
such that
\begin{equation}
\sigma_w^2\Ex\!\left[\left(g_{(k)}^{\lay-1}\right)^2\right]
\le B_k v_k^{\lay-1}+M_k.
\end{equation}
Thus $v_k^{\lay}\le(1+B_k/\dep)v_k^{\lay-1}+M_k/\dep$. Discrete Gr\"onwall gives
\begin{equation}
\label{eq:gronwall}
v_k^{\dep}\le e^{B_k}\left(v_k^0+M_k\right),
\end{equation}
which closes the induction.
\end{proof}

The contrast with \eqref{eq:v2crit} and \eqref{eq:v3growth} is the point. In the plain
network the growth exponent rises with the derivative order, so a fourth-order problem
degrades faster with depth than a second-order one. Under \eqref{eq:resnet} the single
choice $\gamma = 1/2$ bounds each fixed finite order without retuning the exponent. The
constants need not be uniform in derivative order. The pathology is order-dependent; the
scaling exponent is not.

Figure~\ref{fig:residual} compares the residual prediction with finite-width simulations
through order four. The theorem itself holds at every fixed finite order under its stated
assumptions.

\section{Numerical method}
\label{sec:implementation}

\subsection{Exact jet propagation}
\label{sec:impl-jet}

We propagate input derivatives by the Fa\`a di Bruno expansion in
\eqref{eq:faa}, rather than by repeated automatic differentiation. For $\phi=\tanh$, its
first four derivatives are closed-form polynomials in $t=\tanh(z)$:
\begin{align}
\label{eq:tanhderivs}
\phi' &= 1-t^2, & \phi'' &= -2t(1-t^2), \nonumber\\
\phi''' &= -2(1-t^2)(1-3t^2), &
\phi'''' &= 8t(1-t^2)(2-3t^2).
\end{align}
No automatic differentiation is used anywhere, at any order, so the only floating-point
error is roundoff. All computations use double precision.

\subsection{An \texorpdfstring{$O(\NN)$}{O(n)} exact layer sampler}
\label{sec:impl-sampler}

Forming a width-$\NN$ weight matrix costs $O(\NN^2)$ per layer, but the layer update
requires only the products $Wg_{(j)}$. Conditional on the current jet, these products are
jointly Gaussian. For each neuron their covariance is
\begin{equation}
\label{eq:sampler}
\Ex[y_{(j),i}y_{(k),i}]
= \frac{\sigma_w^2}{\NN}\langle g_{(j)},g_{(k)}\rangle=:S_{jk},
\end{equation}
independently across neurons, because distinct rows of $W$ are independent. Drawing
independent samples from $\mathcal N(0,S)$ is therefore distributionally identical
to forming $W$ and multiplying. The resulting sampler costs
$O\!\left(\NN(K+1)^2\right)$ per layer.

This is an exact reformulation and not an approximation: finite-width fluctuations
propagate correctly, because $S$ is computed from the realized $g$ of the previous layer
rather than from its mean-field limit. It is what makes width $\widthMax$ reachable.

\subsection{Estimators}
\label{sec:impl-estimators}

Two choices of estimator matter enough to state, because the natural alternative in each
case destroys the structure being measured.

\paragraph{Polynomial degree by finite differences, not log-log slope.}
The growth laws in Section~\ref{sec:theory} have the form
$v_k^\lay = a + b\,\lay^{\,k-1}$ with $a>0$. The local log-log slope of such a sequence is
$b(k-1)\lay^{k-1}/(a+b\lay^{k-1})$, which approaches $k-1$ only as
$a/(b\lay^{k-1})\to0$ and therefore underestimates the degree at any finite depth. Fitted
slopes over $\lay\in[\fitLo,\depthMain]$ come out at $\loglogKtwo$, $\loglogKthree$ and
$\loglogKfour$ for $k=2,3,4$, against true degrees $1$, $2$ and $3$. We therefore test
degree by the statement that the $(k-1)$-th finite difference of $v_k^\lay$ is constant,
which involves no asymptotic limit.

\paragraph{Covariances centred, not differenced.}
Several quantities in Section~\ref{sec:theory} are departures of an expectation from its
factorized value, for instance
$\sigma_w^2\left(\Ex[\phi'^2h_{(k)}^2] - \Ex[\phi'^2]\,v_k\right)$.
Section~\ref{sec:numerics-assumptions} measures these departures as ratios lying within a
fraction of a percent of unity, so evaluating them as a difference of two separately
estimated means loses most of the significant digits to cancellation. We estimate them as
centred covariances,
$\sigma_w^2\,\overline{\left(\phi'^2-\overline{\phi'^2}\right)\left(h_{(k)}^2-\overline{h_{(k)}^2}\right)}$,
which is algebraically identical and numerically stable.

\subsection{Configuration}
\label{sec:impl-config}

Unless stated otherwise, measurements use width $\NN=\widthMain$, depth
$\dep=\depthMain$, and $\seedsMain$ independently sampled networks, with statistics
averaged over neurons and networks. Fits and depth-averaged quantities are taken over
$\lay\in[\fitLo,\depthMain]$, across all retained depths and across individual networks.

The lower limit excludes an initialization transient. The fixed-point form
\eqref{eq:identities} of Proposition~\ref{prop:identities} requires the preactivation
variance to be constant in $x$, and the first layers are furthest from that: at the input
layer $h^1_{(1)} = W^1$ while $h^1_{(0)} = W^1x + b^1$, so $\sigma_{01}\ne0$ there. The
measured $v_1^\lay$ falls from $\vOneAtOne$ at $\lay=1$ to $\vOneAtFour$ at $\lay=4$ before
flattening.

Discarding those layers is a matter of numerical conditioning rather than of validity.
Lemma~\ref{lem:transient} bounds the discrepancy by $\bar\lambda^{\lay}$ and
Theorem~\ref{thm:asymptotic} carries it through the recursions, so the growth laws do not
require the transient to have ended; the early layers are excluded because they are where
the correction is largest, not because the theory fails there. At $\lay=\fitLo$ the
correction is already of order $\lamTransientWindow$.

Section~\ref{sec:numerics} reports width convergence for each order, which is what
determines the range of $k$ over which the results are claimed.

\subsection{Computational assistance and reproducibility}
\label{sec:impl-assistance}

A large language model (Claude, Anthropic) was used as a computational aid, to expedite
symbolic manipulation in Section~\ref{sec:theory} and to implement the propagation and
sampling code described above. Every derivation reported in this paper was verified
independently by the authors by hand and cross-checked numerically against the measurements
of Section~\ref{sec:numerics}; no algebraic result is included on the basis of that
assistance alone. Every numerical claim is generated by committed code from recorded seeds
rather than transcribed. The archive named under Data availability contains that code
together with the raw outputs and the figure scripts, so each figure and each reported value
can be regenerated rather than taken on trust. Responsibility for all content rests with the
authors.

\section{Parity arguments}
\label{sec:method}

\subsection{Cross terms that parity does not eliminate}
\label{sec:method-parity}

Expanding $\Ex[g_{(j)}g_{(k)}]$ in \eqref{eq:sigmap} produces, alongside the squared terms,
cross terms carrying products of adjacent activation derivatives. For odd $\phi$ each such
product $\phi^{(m)}\phi^{(m+1)}$ is an odd function, and $h_{(0)}$ is symmetric, so
\begin{equation}
\label{eq:parity}
\Ex\!\left[\phi^{(m)}(h_{(0)})\,\phi^{(m+1)}(h_{(0)})\right] = 0 .
\end{equation}
The familiar next step is to conclude that the cross terms vanish. That step is not a
consequence of \eqref{eq:parity}. It is an application of A2, and it fails here.

A cross term has the form $\Ex[\phi^{(m)}\phi^{(m+1)}\,J]$ where $J$ is a product of jet
components. Factorizing it into $\Ex[\phi^{(m)}\phi^{(m+1)}]\,\Ex[J]$ requires the
activation factor to be independent of $J$. It is not: both depend on $h_{(0)}$, the
activation factor directly and the jets through the couplings $\sigma_{0k}$ of
Corollary~\ref{cor:indep}. A product of two mean-zero quantities has zero expectation only
when they are uncorrelated, and these are not.

The size of the resulting error is not marginal. In the order-3 expansion the term
$6\sigma_w^2\Ex[\phi''\phi'h_{(1)}h_{(2)}h_{(3)}]$, which \eqref{eq:parity} would eliminate,
contributes $\bcTerm$ at $\lay = \depthMain$ against a total forcing of $\bcTotal$, or
$\bcShare$ of it.

Two features of this failure are worth separating, because they explain why it is not
already known.

First, order one is immune. The order-1 recursion \eqref{eq:v1} involves only
$\chi_1 = \sigma_w^2\Ex[\phi'^2]$, a square, so no product of adjacent derivatives arises
and \eqref{eq:parity} is never invoked. A derivation that propagates the preactivation
variance and the input-output Jacobian, which is what the mean-field literature computes,
therefore cannot encounter the trap. It appears first at order two.

Second, the failure is invisible to the diagnostic one would naturally use. Assumption A2
is normally checked by comparing an expectation against its factorized value and asking
whether the ratio is one. For these terms the factorized value is exactly zero, so the
ratio is $0/0$ and carries no information. Every other term in \eqref{eq:v2} and
\eqref{eq:v3} passes such a check to within $\aTwoTol$ while these terms are wrong by
construction.

\section{Numerical results}
\label{sec:numerics}

Simulations use the configuration in Section~\ref{sec:implementation}. The reported
intervals are across independently sampled networks. They test the mean-field recursions,
not trained-network accuracy.

\subsection{One hypothesis tested, and one the derivation avoids}
\label{sec:numerics-assumptions}

A1 carries the derivation. Theorem~\ref{prop:gaussianjet} establishes it in the
infinite-width limit at fixed depth, which is a statement about a limit; what the
simulations can add is whether it is accurate at the widths actually used, so it is tested
directly. Skewness and excess
kurtosis of $h_{(1)}$, $h_{(2)}$ and $h_{(3)}$ are all below $\gaussTol$ in magnitude
throughout the measurement window. That $h_{(2)}$ and $h_{(3)}$ are Gaussian is not implied
by the standard mean-field statement, which concerns $h_{(0)}$ alone: both are built from
products through \eqref{eq:faa}. The three Isserlis identities the recursions use hold to
the same precision,
\begin{equation}
\label{eq:isserlis-checks}
\frac{\Ex[h_{(1)}^4]}{3v_1^2} = \isserlisA,
\qquad
\frac{\Ex[h_{(1)}^3h_{(3)}]}{3v_1c_{13}} = \isserlisB,
\qquad
\frac{\Ex[h_{(1)}^2h_{(2)}^2]}{v_1v_2 + 2c_{12}^2} = \isserlisC,
\end{equation}
and the fixed-point identity $\sigma_{01}=0$ of Proposition~\ref{prop:identities} is
measured at $\sigmaOneZero$.

Assumption A2 is a different matter, and the distinction is worth keeping explicit. No
result in Section~\ref{sec:theory} uses it: the recursions at orders two and three follow
from A1 together with the fixed-point identities of Proposition~\ref{prop:identities}, and
every expectation is evaluated by Lemma~\ref{lem:stein} rather than by factorization. A2
nonetheless holds numerically. The decoupling ratios
\begin{equation}
\label{eq:rho}
\rho_m := \frac{\Ex\!\left[\phi^{(m)2}J\right]}{\Ex\!\left[\phi^{(m)2}\right]\Ex[J]}
\end{equation}
for the three terms where they are defined measure $\rhoTwo$, $\rhoOneThree$ and
$\rhoThree$. We report these because their agreement is what makes
Section~\ref{sec:method-parity} precise: A2 is accurate wherever it can be tested, and the
terms it fails on are exactly those where $\Ex[J]=0$ leaves nothing to test.

\begin{figure}[t]
\centering
\includegraphics[width=\linewidth,height=0.78\textheight,keepaspectratio]{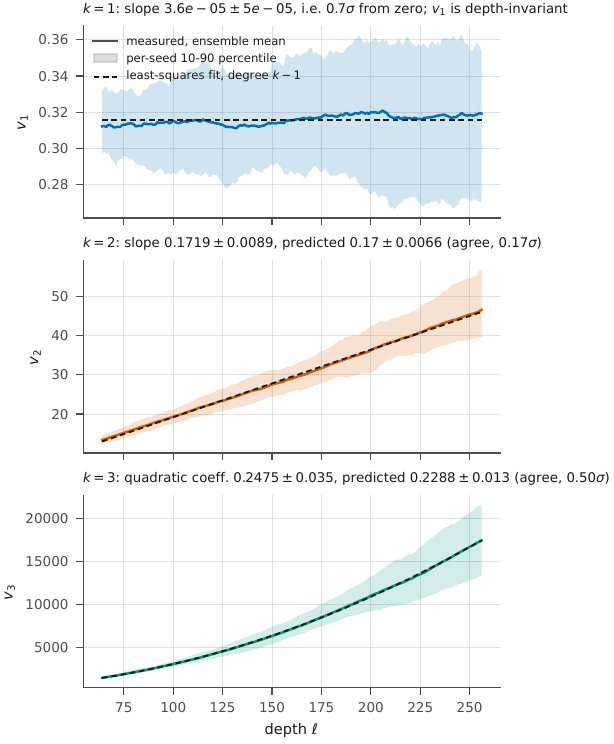}
\caption{Critical depth growth for derivative variances of orders one to three. Solid
curves are ensemble means, shaded bands show the 10th to 90th percentiles, and dashed
curves are the predicted polynomial forms. Width $\NN=\widthMain$, $\seedsMain$
independently sampled networks.}
\label{fig:growth}
\end{figure}

\subsection{Critical growth}
\label{sec:numerics-growth}

Figure~\ref{fig:growth} tests the recursions at $\chio=1$. The order-one slope is
$\slopeOne$, or $\slopeOneSig$ from zero, consistent with \eqref{eq:v1}. The order-two
slope is $\slopeTwo$, while \eqref{eq:v2} predicts $3\chi_2v_1^2=\predTwo$; their difference
is $\devTwo$. The measured order-three quadratic coefficient is $\coefThree$ and the
prediction from \eqref{eq:v3growth} is $\predThree$; their difference is $\devThree$.
Thus the data support depth-invariant first-order variance, linear second-order growth, and
quadratic third-order growth at the measured widths and depths.

\begin{figure}[t]
\centering
\includegraphics[width=\linewidth,height=0.78\textheight,keepaspectratio]{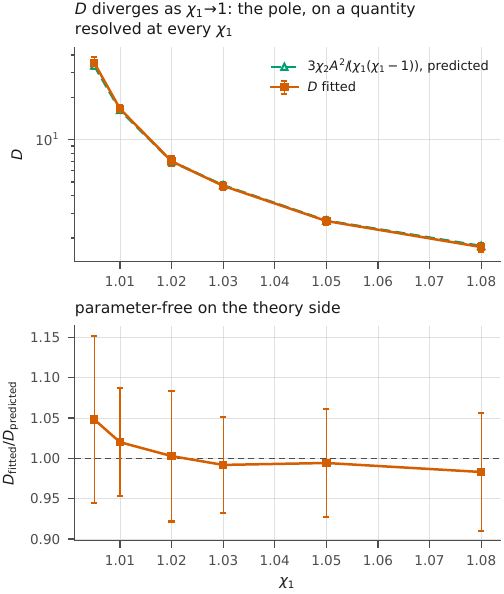}
\caption{The resonance, approached from above. Upper panel: the fitted coefficient $D$ of
$\chi_1^{2\ell}$ against the prediction $3\chi_2A^2/(\chi_1(\chi_1-1))$, which contains no
fitted quantity. Lower panel: their ratio, with intervals covering unity at every $\chi_1$.
The coefficient $C$ is not shown; it is resolved only adjacent to the pole, and the
divergence appears in $D$ alone. Width $\NN=\widthAux$, $\seedsAux$ independently sampled
networks at each $\chi_1$.}
\label{fig:resonance}
\end{figure}

\subsection{The resonance}
\label{sec:numerics-resonance}

Figure~\ref{fig:resonance} tests \eqref{eq:v2sol} above criticality. Fitting
$v_2^{\lay} = C\chio^{\lay} + D\chio^{2\lay}$ over the measurement window at each $\chio$
and comparing the fitted $D$ against $3\chi_2A^2/(\chio(\chio-1))$ gives a ratio within
$\dRatioTol$ of unity at every $\chio$ measured, with intervals covering unity throughout.
Nothing on the theory side of this comparison is fitted: $\chi_2$ follows from the
activation and the fixed point, and the amplitude $A$ is taken from $v_1$.

We do not report $C$. It is resolved only for $\poleValidRange$; beyond that
$\chio^{\lay}$ is negligible against $\chio^{2\lay}$ across the window, the two basis
functions cease to be separable, and the fitted $C$ is consistent with zero. $D$ suffices,
since the divergence appears in $D$ alone.

Below $\chio = 1$ the signals decay and the estimator is dominated by its noise floor, so
the pole is approached from above only.

\begin{figure}[t]
\centering
\includegraphics[width=\linewidth,height=0.78\textheight,keepaspectratio]{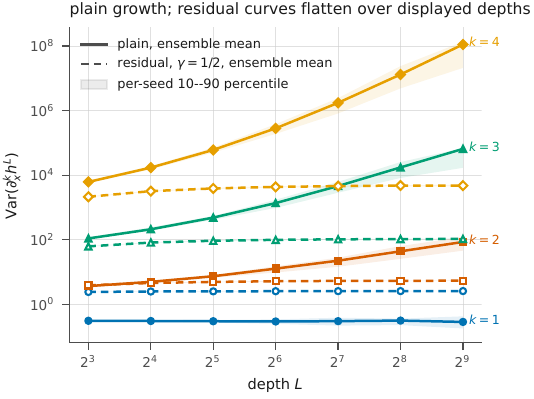}
\caption{Derivative variances for a plain network and a residual network with
$L^{-1/2}$ branch scale, orders one to four. Lines are ensemble means and shaded bands show
the 10th to 90th percentiles. Width $\NN=\widthAux$, $\seedsAux$ independently sampled
networks at each depth.}
\label{fig:residual}
\end{figure}

\subsection{Residual scaling}
\label{sec:numerics-residual}

Figure~\ref{fig:residual} compares the plain network with \eqref{eq:resnet} at
$\gamma=1/2$ from depth $8$ to $512$. In the plain network, the order-two, order-three,
and order-four variances increase by factors of $\plainGrowTwo$, $\plainGrowThree$, and
$\plainGrowFour$. Under residual scaling, the corresponding factors are $\resGrowTwo$,
$\resGrowThree$, and $\resGrowFour$. The result is consistent with the boundedness claim
in Proposition~\ref{prop:residual}. Order four is shown only as a boundedness check; this
paper does not claim an order-four mean-field growth law.

\begin{figure}[t]
\centering
\includegraphics[width=\linewidth,height=0.78\textheight,keepaspectratio]{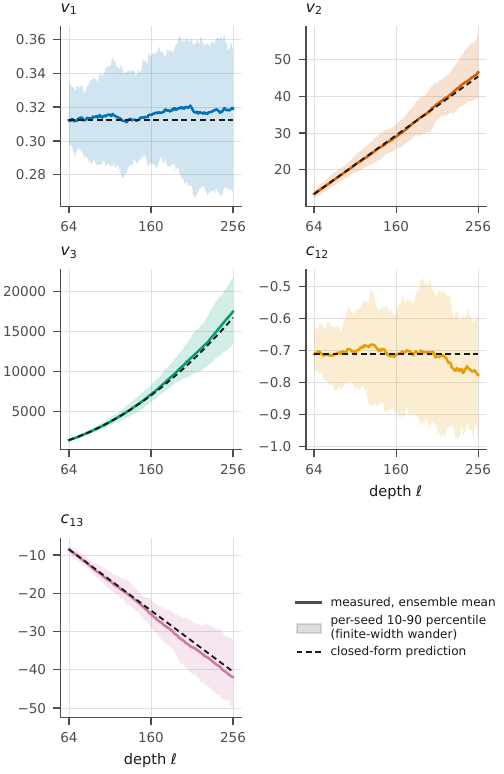}
\caption{Closed third-order recursion integrated from $\ell=64$. Solid curves are
ensemble measurements and dashed curves are predictions from
$(v_1,v_2,v_3,c_{12},c_{13})$ with no later fitted input. Width $\NN=\widthMain$,
$\seedsMain$ independently sampled networks; the tighter closure test quoted in the text
is a separate run.}
\label{fig:closure}
\end{figure}

\subsection{Closed recursion}
\label{sec:numerics-closure}

Figure~\ref{fig:closure} integrates \eqref{eq:v1}, \eqref{eq:v2}, \eqref{eq:v3}, and
\eqref{eq:offdiag} from $\lay=\fitLo$ using $\chio$, $\chi_2$, and $\chi_3$. In a separate
run at the same width $\NN=\widthMain$ with $\seedsClosure$ networks, all five ensemble
ratios are within $\closureBand$ of one. This test uses no measured quantity after the
initial condition.
The agreement supports the closed covariance system through third order at the measured
width and depth range.

\subsection{Width dependence of the residual}
\label{sec:numerics-width}

The closure residual is a finite-width effect. Across three widths at $\seedsClosure$
networks each, the largest deviation from unity falls from $\closMaxDevSmallN$ at
$\NN = 2^{14}$ to $\closMaxDevMidN$ at $2^{16}$, then changes little, reading
$\closMaxDevLargeN$ at $2^{18}$. Fitting a power law to the three widths gives exponents of
$\expVtwo$, $\expVthree$ and $\expCthirteen$ for the three quantities whose deviations are
resolved, with large fit residuals. We report these but do not treat them as a rate: the
$2^{16}$ and $2^{18}$ deviations sit at one to two standard errors, so the fit is
constrained by a single resolved step. The statement the data supports is that the residual
falls sharply with width and then plateaus near $\plateauBand$.

A plateau that does not move with width requires a source that is not finite width. One
candidate is ruled out at once. The finite-depth correction to the fixed-point identities,
bounded by Lemma~\ref{lem:transient} and carried through the recursions by
Theorem~\ref{thm:asymptotic}, has fallen to order $\lamTransientWindow$ by the first layer
of the window, which is seventeen orders of magnitude below the residual and cannot account
for it. A second candidate does account for its size. The integration uses $\chio$ evaluated by quadrature at
$q^\ast$, while a realized network has its own effective multiplier; a relative discrepancy
$\delta$ compounds over the window, so $\delta \sim 10^{-4}$ produces a deviation of a few
per cent that does not shrink with $\NN$. The measured per-layer discrepancy is
$\driftDelta$, which is the right size.

We do not regard this as established, for three reasons. The discrepancy is $\slopeOneSig$
from zero, so it is not itself resolved. The ordering of deviations across quantities
matches what compounding $\chio$ error would produce at $2^{14}$ and $2^{16}$ but not at
$2^{18}$. And $v_1$ and $c_{12}$ obey the same homogeneous recursion, so a common $\chio$
error must imply one value of $\delta$ from both; the two implied values differ by
$\driftMismatch$, which is the consistency test the explanation has to pass and it passes
weakly.

\section{Discussion}
\label{sec:discussion}

The practical reading of Proposition~\ref{prop:dichotomy} is that tuning cannot solve the
problem it describes. Choosing $(\sigma_w^2, \sigma_b^2)$ selects a point on the
one-parameter family of variance maps, and Section~\ref{sec:theory-growth} shows that the
point which stabilizes the first derivative is the point at which the second grows linearly.
Moving off criticality trades that linear growth for exponential behavior in the first
derivative, which is the failure the criterion was introduced to prevent. For a loss that
depends on second or higher derivatives there is no setting of the two variances that avoids
both.

That leaves the architecture as the only place to intervene, which is what
Proposition~\ref{prop:residual} formalizes. The relevant feature of \eqref{eq:resnet} is not
the skip connection itself but the attenuation of the branch: the factor $\dep^{-1/2}$ makes
each layer's contribution small enough that the accumulated effect over $\dep$ layers stays
bounded, and \eqref{eq:gronwall} holds at every order with the same exponent. An
architecture that attenuates layer contributions in some other way, or that begins near
shallow and deepens during training as the trainable branch scalar of
Wang et al.~\citep{wang2024pirate} does, would be expected to have the same effect for the
same reason.

Three directions remain outside the scope of what is established here. The analysis is of
the network at initialization, and the corresponding statement about training dynamics
requires following these quantities under gradient descent rather than at a single point.
The derivation assumes a smooth activation, and Proposition~\ref{prop:dichotomy} shows that
the non-smooth case is not a limit of it but a separate regime in which the second derivative
vanishes identically. And the input is scalar throughout; for a multidimensional input the
jet becomes a collection of mixed partial derivatives, and whether the closure of
Corollary~\ref{cor:closure} survives that generalization is not something the present
argument settles.

\section{Conclusions}
\label{sec:conclusions}

The covariance of the input derivatives of a wide scalar-input network obeys mean-field
recursions through third order, exact at the variance fixed point and with finite-depth
corrections that decay geometrically at a rate set by the variance map rather than by the
edge-of-chaos condition. The joint Gaussianity they use is not an extra hypothesis: it holds
in the infinite-width limit at each fixed depth. Their coefficients are the susceptibilities $\chio$, $\chi_2$ and $\chi_3$ together
with the fixed-point identities of Proposition~\ref{prop:identities}, which hold because the
variance fixed point does not depend on the input. Every remaining moment of the activation
cancels, so the five-quantity system closes on those three susceptibilities alone, with no
measured quantity carried along to keep it closed.

At the edge of chaos that closure has a definite shape. The condition $\chio=1$ is what
holds the first-derivative variance constant through depth, and it is the same condition
that makes the second-order recursion resonant: the two modes $\chio^{\lay}$ and
$\chio^{2\lay}$ of its solution coincide there, and their coincidence converts a bounded
solution into one growing linearly in depth whenever the activation has nonzero curvature.
The third-order variance grows quadratically. The two behaviors are therefore not
independent settings to be traded against one another but consequences of the same
condition, which is the content of Proposition~\ref{prop:dichotomy}. Simulations confirm the
growth exponents, the parameter-free amplitude of the resonance, the bound of
Proposition~\ref{prop:residual} and the closed recursion.

Two questions are left open. Conjecture~\ref{conj:closure} asserts that the closure persists
at every order, with coefficients polynomial in $\chi_1,\dots,\chi_k$ and no other moment of
the activation. The evidence is orders one through three, where each cancellation consumes
one of the fixed-point identities, and the general case is unproved. The order-four variance
is reported only as a boundedness check; no order-four growth law is claimed here.

%

\section*{CRediT authorship contribution statement}

\textbf{Prashant Singh}: Conceptualization, Formal analysis, Investigation, Methodology,
Software, Validation, Visualization, Writing -- original draft, Writing -- review and
editing. \textbf{Pranav Singh}: Conceptualization, Formal analysis, Investigation,
Methodology, Software, Validation, Visualization, Writing -- original draft, Writing --
review and editing. Both authors contributed equally to this work.

\section*{Declaration of competing interest}

The authors declare that they have no known competing financial interests or personal
relationships that could have appeared to influence the work reported in this paper.

\section*{Funding}

This research did not receive any specific grant from funding agencies in the public,
commercial, or not-for-profit sectors.

\section*{Data availability}
Code, the exact layer sampler, raw sweep outputs and the figure-generation scripts will be
made publicly available at \url{\zenodoURL} (\zenodoDOI) on publication. The archived
release will identify the source revision, software environment and random seeds needed to
regenerate each figure and reported number.

%
%

\section*{Declaration of Generative AI and AI-assisted technologies in the writing process}

During the preparation of this work the authors used Claude (Anthropic) in order to draft
and revise manuscript text and to assist with LaTeX preparation. After using this tool the
authors reviewed and edited the content as needed and take full responsibility for the
content of the published article.

\bibliographystyle{elsarticle-num}
\bibliography{refs}

\end{document}